\documentclass[12pt]{article}
\usepackage{geometry}                
\usepackage{setspace}
\usepackage[parfill]{parskip}    
\usepackage{newtxtext}
\usepackage[round]{natbib}

\usepackage{amsmath}
\usepackage{amsthm}
\usepackage{amssymb}
\usepackage{mathtools}
\usepackage{bm}

\newtheorem{theorem}{Theorem}
\newtheorem{corollary}{Corollary}
\newtheorem{lemma}{Lemma}
\newtheorem{definition}{Definition}
\newtheorem{assumption}{Assumption}
\newtheorem{remark}{Remark}

\usepackage{graphicx}
\usepackage{subfig}
\usepackage{booktabs}       

\usepackage[hidelinks]{hyperref}
\usepackage{url}

\usepackage{fancyhdr}
\usepackage{color}

\title{Super-model ecosystem:\\A domain-adaptation perspective}

\author{Fengxiang He\thanks{The authors are with JD Explore Academy, JD.com Inc., Beijing, 100176, China. Email: \href{mailto:fengxiang.f.he@gmail.com}{fengxiang.f.he@gmail.com} and \href{mailto:dacheng.tao@gmail.com}{dacheng.tao@gmail.com}.} \and Dacheng Tao\footnotemark[1]}

\date{}

\begin{document}

\maketitle

\begin{abstract}
This paper attempts to establish the theoretical foundation for the emerging super-model paradigm via domain adaptation, where one first trains a very large-scale model, {\it i.e.}, super model (or foundation model in some other papers), on a large amount of data and then adapts it to various specific domains. Super-model paradigms help reduce computational and data cost and carbon emission, which is critical to AI industry, especially enormous small and medium-sized enterprises. We model the super-model paradigm as a two-stage diffusion process: (1) in the pre-training stage, the model parameter diffuses from random initials and converges to a steady distribution; and (2) in the fine-tuning stage, the model parameter is transported to another steady distribution. Both training stages can be mathematically modeled by the Uhlenbeck-Ornstein process which converges to two Maxwell-Boltzmann distributions, respectively, each of which characterizes the corresponding convergent model. An $\mathcal O(1/\sqrt{N})$ generalization bound is then established via PAC-Bayesian framework. The theory finds that the generalization error of the fine-tuning stage is dominant in domain adaptation. In addition, our theory suggests that the generalization is determined by a new measure that characterizes the domain discrepancy between the source domain and target domain, based on the covariance matrices and the shift of the converged local minimum.
\end{abstract}


\section{Introduction}

Large-scale pretrained models (or foundation models) \citep{han2021pre, chen2021pre}, including GPT-3 \citep{brown2020language} and BERT \citep{devlin2018bert}, enables a new paradigm in machine learning: pre-training a large-scale model on very large-scale datasets and then transferring the learned model to an unseen domain. This paradigm was first introduced in natural language processing and recently to computer vision. It sheds light in a higher-level automation and is establishing a new paradigm with advantages/supremacy. In this paradigm, a super model learns meta knowledge from large amounts of data and reduces learning cost in specific domains. This helps considerably reduce the computational and data cost of applying machine learning in many specific applications. This is thus of significant values to enormous small and medium-sized enterprises. Additionally, super-model paradigm enables better management of the geographic location of machine learning workload and the datacenter infrastructure, which has been shown able to significantly reduce the carbon emission \citep{patterson2021carbon}.

Technically, domain adaptation plays a vital role for the knowledge transferring in the super-model paradigm. Usually, the data in a target domain is much smaller than the one in the source domain. In the light of this, an appropriate understanding to the generalizability of the transferred model on the target domain is of high importance.

In this paper, we prove an upper bound for the generalization error (generalization bound) for domain adaptation algorithms. The generalization error is defined as the difference between the expected risk $\mathcal R$ and the empirical risk $\hat{\mathcal R}$. Intuitively, a larger generalization bound indicates that the generalization error is possibly larger and thus suggests worse generalizability.

We model the super model paradigm as a two-stage diffusion processes. In the first stage, stochastic gradient-based optimizers, usually stochastic gradient-based optimization,  including stochastic gradient descent (SGD) \citep{robbins1951stochastic}, momentum \citep{nesterov1983method, tseng1998incremental}, and Adam \citep{kingma2014adam}, learns a pre-trained model on the source-domain data via empirical risk minimization,
\begin{equation*}
\min_{\theta} \hat{\mathcal R}_S(\theta) = \min_{\theta} \frac{1}{N} \sum_{i=1}^N\ell(h_{\theta} (x_i), y_i),
\end{equation*}
where $\hat{\mathcal R}_S(\theta)$ is the empirical risk of model parameterized by $\theta$ on the training sample $S$, which is defined to be
\begin{equation}
S = \{(x_1, y_1), \ldots, (x_N, y_N) | x_i \in \mathbb R^{d_X}, y_i \in \mathbb R^{d_Y}\},
\end{equation}
where $N$ is the training sample size, and $l$ is the loss function. The convergent parameter initializes the fine-tuning process on the target domain in the second stage. We model the parameter trajectory of SGD by a stochastic process, Uhlenbeck-Ornstein process \citep{uhlenbeck1930theory}, as follows,
\begin{align}
	\Delta \theta(t) = & \theta(t + 1) - \theta(t) = - \eta \hat{g}_{S}(\theta(t)) \nonumber\\
	= & - \eta g(\theta) + \frac{\eta}{\sqrt{|S|}}
	B \Delta W,~ \Delta W \sim \mathcal N(0, I),
\end{align}
where $B$ is positive definite matrix which characterizes the covariance of the gradient noise. 
This can also be smoothed to a diffusion equation, Fokker-Plank equation. Correspondingly, the trajectories can be modeled by the dynamics of the Fokker-Plank equations. 
Further, the steady distributions of the Uhlenbeck-Ornstein equations characterize the distributions of the learned models.

Deep learning can be formulated as solving a non-convex optimization problem: The loss surface of neural networks are usually highly non-convex due to the complexity of neural network architectures. In general, solving a non-convex optimization problem is NP-hard. However, numerous experiments show that deep learning has excellent optimization performance. This mystery is partially addressed by some empirical finding on the local convexity and smoothness of the loss surfaces of deep neural networks. Empirical results show that the loss surface around the convergent local minima is second-order smooth, as shown by \citet{li2018visualizing}. 

This empirical finding inspires us to model the loss surface around the convergent local minimum as a quadratic function. This assumption determines the derivatives and boundary conditions of the Fokker-Plank equation. Moreover, the model parameter is usually initialized by following a Gaussian distribution. Based on them, Fokker-Plank equation has a steady distribution in the form of Maxwell-Boltzmann distribution, which governs the distribution of the learned model by the SGD. During the pre-training stage, SGD converges a Maxwell-Boltzmann distribution around the local minimum given below, 
\begin{equation}
	q_{PT}(\theta) = M_{PT} \exp \left\{ -\frac{1}{2}\theta^{\top} \Sigma_	{PT}^{-1} \theta \right\},
\end{equation}
where $M_{PT}$ is the normalizer and $\Sigma_{PT}$ is the covariance.

This distribution is then used as the initial distribution in the fine-tuning stage. Subsequently, SGD in the fine-tuning stage learns the mapping from the initialization to a new Maxwell-Boltzmann distribution centered at the new local minimum on the loss surface in the target domain as follows,
\begin{equation}
	q_{FT}(\theta) = M_{PT} \exp \left\{ -\frac{1}{2}(\theta - \theta_{FT})^{\top} \Sigma_
	{FT}^{-1} (\theta - \theta_{FT}) \right\},
\end{equation}
where $M_{FT}$ is the normalizer, $\theta_{FT}$ is the distribution shift, and $\Sigma_{FT}$ is the covariance.

Based on the diffusion processes, we then establish PAC-Bayesian generalization bounds for the learned super model on the source domain and the transferred model on the target domain. The PAC-Bayesian framework \citep{mcallester1999pac, mcallester1999some} upper bounds the generalization error of a stochastic algorithm via the distance between the initial distribution and the distribution of the learned hypothesis, usually measured by some information-theoretical distances, such as KL-divergence. Intuitively, the PAC-Bayesian theory suggests that training a very-large model from a no-knowledge prior, such as Gaussian distribution and uniform distribution, needs a very large amount of data to secure the generalizability; and if the initialization is near the distribution of the learned hypothesis, the needed sample complexity can be relatively much smaller. However, a high-quality prior is not accessible in practice. This significantly limits the model size, particularly in low-resource scenarios. This renders the key motivation of the super-model paradigm: (1) training a super model on a very large-scale dataset, in order to learn a high-quality model from the no-knowledge prior; and (2) using the learned super model as a high-quality prior in the down-stream application, in order to reduce the needed training data and supports larger model size.

In this paper, the generalization bound in pre-training is established based on the KL-divergence between the Maxwell-Boltzmann distribution in pre-training as below,
	\begin{align}
	\label{eq:thm:PAC-Bayesian_SGD-PT}
		& R(Q_{PT}) \le \hat R(Q_{PT}) \nonumber\\
		+ & \sqrt{\frac{D(Q_{PT}, P) + 2 \log \left( \frac{1}{\delta} \right) + 2 \log N_{PT} + 4}{4N_{PT} - 2}},
	\end{align}
	where
\begin{align*}
D(Q_{PT}, P) 
= \log \left({\det(\Sigma_{PT})} \right) + \text{tr}(\Sigma_{PT} - I),
\end{align*}
and $R(Q_{PT})$ is the expected risk, $\hat R(Q_{PT})$ is the empirical risk, $\Sigma_{PT}$ is the covariance of the distribution of the learned hypothesis, and $N_{PT}$ is the training sample size in the pre-training.

Meanwhile, the generalization bound in pre-training is established based on the KL-divergence between the two Maxwell-Boltzmann distributions in pre-training and fine-tuning as follows,
	\begin{align}
	\label{eq:thm:PAC-Bayesian_SGD-FT}
		& R(Q_{FT}) \le \hat R(Q_{FT}) \nonumber\\
		+ & \sqrt{\frac{D(Q_{FT}, Q_{PT}) + 2 \log \left( \frac{1}{\delta} \right) + 2 \log N_{FT} + 4}{4N_{FT} - 2}},
	\end{align}
	where
\begin{align*}
& D(Q_{FT}, Q_{PT}) \nonumber\\
= & \log \left({\det(\Sigma_{PT}^{-1}\Sigma_{FT})} \right) + \text{tr}(\Sigma_{PT}^{-1}\Sigma_{FT} - I) + \theta_{FT}^\top \Sigma_{PT}^{-1}\theta_{FT},
\end{align*}
and $R(Q_{FT})$ is the expected risk, $\hat R(Q_{FT})$ is the empirical risk, $\Sigma_{FT}$ is the covariance of the distribution of the learned hypothesis, $\theta_{FT}$ is the shift of the distribution center, and $N_{FT}$ is the training sample size in the fine-tuning.

We further define two new notions to measure the domain discrepancy as follows,
\begin{align*}
& D(Q_{FT}, Q_{PT}) \nonumber\\
= & \log \left({\det(\Sigma_{PT}^{-1}\Sigma_{FT})} \right) + \text{tr}(\Sigma_{PT}^{-1}\Sigma_{FT} - I) + \theta_{FT}^\top \Sigma_{PT}^{-1}\theta_{FT},
\end{align*}
and
\begin{align*}
& \tilde D(Q_{FT}, Q_{PT}) \nonumber\\
= & \log(\text{tr}(\Sigma_{PT}^{-1}\Sigma_{FT})) + \text{tr}(\Sigma_{PT}^{-1}\Sigma_{FT}) + \theta_{FT}^\top \Sigma_{PT}^{-1}\theta_{FT} \nonumber\\
& + d \log d - d,
\end{align*}
where $d$ is the parameter size. These two notions measure the magnitude of the domains shifts based on the learned hypotheses on the source domains and target domains.

Our theory have the following implications:
\begin{itemize}
\item
The very large-scale datasets employed in the pre-training stage helps secure obtaining a high-quality model from a no-knowledge prior. The learned model carries the knowledge learned from the training data in the pre-training stage. The distribution of the pre-trained model severs as a high-quality initialization in the down-stream fine-tuning stage.

\item
Comparing the generalization bounds in the pre-training stage and fine-tuning stage, we show that the generalization error of the fine-tuning stage is dominant in the super-model paradigm. This is because the dominantly large size of the training data in the pre-training stage. This finding supports the feasibility and efficiency of the super-model paradigm.

\item
Our generalization bound suggests that the generalization on the target domain is determined by the magnitude of the domain shifts. Generally, larger domain shifts lead to worse generalization on the target domain. This supports the heuristic in practice that the performance on the target domain is limited by the domain shifts.
\end{itemize}

It is worth noting that the super-model paradigm also supports model compression in the model deployment, including pruning, quantization, model distillation, etc. The influence of model compression methods can be directly plugged in our theory.


\section{Background}
\label{sec:review}

This section reviews the related work, including super model, domain adaptation, generalization, and deep learning theory.

\textbf{Domain adaptation.} Domain adaptation algorithms transfer knowledge from one domain to another. It enables the super model paradigm. Domain adaptation has three main streams:

(1) Discrepancy-based domain adaptation modifies the loss to narrow the discrepancy between the features from the source domain $\mathcal{D}^s$ and the ones from the target domain $\mathcal{D}^t$. \citet{tzeng2014deep} introduce a fully-connected adaptation layer into CNN for learning the representation of the kernel $\phi(\cdot)$ in order to minimize the maximum mean discrepancy (MMD) between the features from different domains:
\begin{align*}
& \operatorname{MMD}\left(X_{S}, X_{T}\right) \nonumber\\
=& \left\|\frac{1}{\left|X_{S}\right|} \sum_{x_{s} \in X_{S}} \phi\left(x_{s}\right)-\frac{1}{\left|X_{T}\right|} \sum_{x_{t} \in X_{T}} \phi\left(x_{t}\right)\right\|.
\end{align*}
\citet{long2015learning} employ multiple adaptation layers. \citet{long2015learning, long2016unsupervised} introduce residual blocks into the classifiers of source domain. \citet{long2017deep} further consider the discrepancy of the joint distribution $P(X,Y)$ rather than the marginal distribution $P(X)$;

(2) Adversarial-based domain adaptation maps the source domain $\mathcal{D}^s$ and the target domain $\mathcal{D}^t$ to a general space, inspired by generative adversarial networks (GANs): If a classifier hardly separates examples of source domain and those from target domain, the feature extractor has narrowed the two domains. \citet{ganin2015unsupervised}  propose gradient reversal layers that reverse the gradients generated by the domain classifier during backpropagation. \citet{zhang2018collaborative} argue that the features of the bottom layers contain more domain information, 
while those of the top layers contain less domain information. 
They further employ collaborative learning to learn domain informative features in the bottom layers, and adapt adversarial learning to learn domain uninformative features in the top layers; and

(3) Reconstruction-based domain adaptation reconstructs the features extracted from the source domain $\mathcal{D}^s$ to the target domain $\mathcal{D}^t$. \citet{ghifary2016deep} reconstruct examples from the target domain via features learned from source domain classification task. \citet{bousmalis2016domain} reconstruct the inputs via both private representation and shared representation of both source and target domains.

\textbf{Generalization.} Good generalization guarantees that an algorithm learns the underlying patterns in training data rather than just memorize the data. In this way, good generalization abilities provide confidence that the models trained on existing data can be applied to similar but unseen scenarios. Three major approaches in analyzing the generalizability are seen in the literature: (1) generalization bounds based on the hypothesis complexity, including VC dimension \citep{blumer1989learnability, vapnik2006estimation}, Rademacher complexity \citep{koltchinskii2000rademacher, koltchinskii2001rademacher, bartlett2002rademacher}, and covering number \citep{dudley1967sizes, haussler1995sphere}. The results are usually obtained via concentration inequalities. They also suggest controlling the model size to secure the generalizability, which is no longer valid in deep learning; (2) generalization bounds based on the algorithmic stability \citep{rogers1978finite, bousquet2002stability, xu2011sparse}. The results in this stream follow the motivation that learning algorithms robust to small disturbances in input data usually have good generalizability; and (3) generalization bounds in the PAC-Bayes framework \citep{mcallester1999pac, mcallester1999some}. The results are obtained based on information-theoretical versions of concentration inequalities.

\textbf{Deep learning theory.} Deep learning has been deployed successfully in many real-world scenarios. However, the theoretical foundations of deep learning are still elusive. For example, there is no explanation for how deep learning algorithms work, why they can succeed, when they would fail, and whether they would hurt society. Such deficiency in explainability questions the transparency and accountability of deep learning, and further undermines our confidence of deploying deep learning in security-critical application domains, such as medical diagnosis \citep{kulikowski1980artificial, silver2016mastering} and drug discovery \citep{chen2018rise}. Many works have emerged to establish the theoretical foundations of deep learning via VC dimension \citep{harvey2017nearly}, Rademacher complexity \citep{golowich2017size, bartlett2017spectrally}, covering number \citep{bartlett2017spectrally}, Fisher-Rao norm \citep{liang2019fisher, tu2020understanding}, PAC-Bayesian framework \citep{neyshabur2017pac}, algorithmic stability \citep{hardt2016train, kuzborskij2018data, verma2019stability}, and the dynamics of stochastic gradient descent or its variants \citep{mandt2017stochastic, mou2018generalization, he2019control}. Please see more related works in surveys \citep{e2020towards, he2020recent, poggio2020theoretical}. This work is committed to establishing theoretical foundations of privacy, generalization, adversarial attack in deep learning, all of which have profound importance in enhancing the explainability, transparency, and accountability of deep models.

\textbf{Generalization of SGD.} Some generalization bounds for algorithms trained by SGD are proposed. \citet{mou2017generalization} analyze the generalization of stochastic gradient Langevin dynamics (SGLD), and prove an $\mathcal O\left(1/N\right)$ upper bound and an $\mathcal O\left(1/\sqrt{N}\right)$ upper bound for the generalization error, respectively via algorithmic stability and PAC-Bayesian theory. \citet{pensia2018generalization} analyze the generalizability of noisy and iterative machine learning algorithms. A generalization bound is then proved given the mutual information between the output hypothesis and the input data. It also proved generalization bounds for SGLD as examples. \citet{chen2018stability} prove that the convergence and stability for iterative machine learning algorithms have a trade-off under both convex smooth assumption and strong convex smooth assumption. Under the same assumptions, \citet{chen2018stability} prove an $\mathcal O\left(1/N\right)$ generalization bound for SGD. \citet{liu2017algorithmic} prove an $\mathcal O (1/N)$ generalization bound for SGD when the loss function is Lipschitz continuous and smooth. \citet{london2017pac} prove a generalization bound for SGD based on the KL divergence between the prior $P$ and the posterior $Q$ under the PAC-Bayes framework. \citet{he2019control} present a PAC-Bayes generalization bound for SGD based on stochastic differential equations. In the work of He {\it et al.}, the gradient noise is modeled by a Gaussian distribution. \citet{meng2020dynamic} extend the gradient noise to be state-dependent. \citet{cheng2020stochastic} extend the gradient noise to be Levy process.

\section{Notations and preliminaries}
\label{sec:preliminaries}

Suppose the training dataset is $S = \{(x_1, y_1), \ldots, (x_N, y_N) | x_i \in \mathbb R^{d_X}, y_i \in \mathbb R^{d_Y}, i = 1, \ldots, N\}$, where $d_X$ is the dimension of the feature $X$ and $d_Y$ is the dimension of the label $Y$. Suppose $x_i$ and $y_i$ are independent and identically distributed (i.i.d.) observation of variables $X \in \mathcal X$ and $Y \in \mathcal Y$, respectively. We also rewrite $z_i = (x_i, y_i)$, which is an i.i.d. observation of random variable $Z = (X, Y) \in \mathcal Z$. Denote the generating distribution of $Z$ is $\mathcal D$.

Formally, machine learning algorithms are designed to select the hypothesis function $F_{\theta}$ with the lowest expected risk $\mathcal R$ under the loss function $l$ from a hypothesis class $\{ F_{\theta} | \theta \in \Theta \subset \mathbb R^{d} \}$, where $\theta$ is the parameter of the hypothesis and $d$ is the dimension of the parameter $\theta$. For many stochastic algorithms, such as SGD, we usually use a distribution to express the output parameter. Suppose the parameter follows a distribution $Q$, the expected risks respectively in terms of $\theta$ and $Q$ are defined as:
\begin{gather}
\label{eq:parametric_expected_risk}
	\mathcal R(\theta) =\mathbb E_{(X, Y) \sim \mathcal D} l(F_{\theta}(X), Y),\\
	\mathcal R(Q) = \mathbb E_{\theta \sim Q} \mathbb E_{(X, Y) \sim \mathcal D} l(F_{\theta}(X), Y).
\end{gather}
However, the expected risk $\mathcal R$ is not available from the data, since we do not know the formulation of latent distribution $\mathcal D$ of data. Practically, we use the empirical risk $\hat{\mathcal R}$ to estimate the expected risk $\mathcal R$, which is defined as:
\begin{gather}
	\hat{\mathcal R}(\theta) = \frac{1}{|T|} \sum_{i = 1}^{|T|} l(F_{\theta}(X_{i}), Y_{i}),\\
	\hat{\mathcal R}(Q) = \mathbb E_{\theta \sim Q} \left[ \frac{1}{|T|} \sum_{i = 1}^{|T|} l(F_{\theta}(X_{i}), Y_{i}) \right],
\end{gather}
where all $(X_{i}, Y_{i})$ constitute the training sample $T$.

Learning algorithms usually solve the following empirical risk minimization (ERM) problem to approach the optimal hypothesis,
\begin{equation*}
\min_{\theta} \hat{\mathcal R}_S(\theta) = \min_{\theta} \frac{1}{N} \sum_{i=1}^N\ell(h_{\theta} (x_i), y_i).
\end{equation*}
We usually employ stochastic gradient-based optimizers for ERM in deep learning. Popular options of stochastic gradient-based optimizers include stochastic gradient descent (SGD) \citep{robbins1951stochastic}, momentum \citep{nesterov1983method, tseng1998incremental}, and Adam \citep{kingma2014adam}. For the brevity, we analyze SGD in this paper. The analysis for other stochastic gradient-based optimizers is similar.

Suppose $\mathcal B$ is a mini batch randomly drawn from the training sample set $S$. Then, the stochastic gradient on $\mathcal B$ is as follows,
\begin{equation*}
\hat g^\mathrm{ERM} (\theta) = \frac{1}{|\mathcal{B}|} \sum_{(x_i, y_i) \in \mathcal B} \nabla_{\theta}\ell(h_{\theta} (x_i), y_i).
\end{equation*}
In the $t$-th iteration, the weight is updated as follows,
\begin{equation*}
\theta^\mathrm{ERM}_{t+1} = \theta^\mathrm{ERM}_t - \eta_t \hat g^\mathrm{ERM} (\theta^\mathrm{ERM}_t),
\end{equation*}
where $\theta^\mathrm{ERM}_t$ is the weight vector in the $t$-th iteration and $\eta_t$ is the corresponding learning rate.

Meanwhile, adversarial training employs SGD to solve the following minimax problem,
\begin{equation}
\label{eq:adversarial_risk}
\min_{\theta} \hat{\mathcal R}^A_S(\theta) = \min_{\theta} \frac{1}{N} \sum_{i=1}^N \max_{\Vert x_{i}^\prime - x_{i} \Vert \leq \rho}\ell(h_{\theta} (x_i^\prime), y_i),
\end{equation}
where $\rho$ is the radius of the ball centered at the example $(x_i, y_i)$. Here, we call $\hat{\mathcal R}^A_S(\theta)$ adversarial empirical risk. Correspondingly, the stochastic gradient on a mini batch $\mathcal B$ and the weight update are calculated as below,
\begin{gather}
\label{eq:adversarial_gradient}
\hat g^A (\theta) = \frac{1}{|\mathcal B|} \sum_{(x_i, y_i) \in \mathcal B} \nabla_{\theta} \max_{\Vert x_{i}^\prime - x_{i} \Vert \leq \rho}\ell(h_{\theta} (x_i^\prime), y_i), \nonumber\\
\theta^A_{t + 1} = \theta^A_t - \eta_t \hat g^A (\theta^A_t).
\end{gather}

\begin{definition}[KL Divergence; cf. \citet{kullback1951information}]
        \label{def:KL_divergence}
        Suppose two distributions $P$ and $Q$ are defined on the same support. Then the KL divergence between $P$ and $Q$ is defined as
\begin{equation*}
    D_{KL}(P\Vert Q) =\mathbb E_P \left(\log\frac{\text d P}{\text d Q}\right).
\end{equation*}
\end{definition}

To avoid technicalities, the measurability/integrability issues are ignored throughout this paper. Moreover, Fubini's theorem is assumed to be applicable for any integration with respect to multiple variables, that the order of integrations is exchangeable. Also, we assume the stable (stationary) solutions of all stochastic differential equations involved exit and are unique.

\section{Super-model paradigm}
\label{sec:super-model-ecosystem}

A supreme industrial paradigm has been emerging that (1) pre-training a large-scale model on large amounts of multi-modality data, such as GPT-3 \citep{brown2020language} and BERT \citep{devlin2018bert}; and (2) fine-tuning the obtained model on specific smaller domain where data size is relatively difficult to access. In this paper, we name it as {\it super-model paradigm.} Super-model paradigm enables efficient and effective knowledge discovery in low-resource application scenarios, including few-shot learning \citep{snell2017prototypical, sung2018learning} and zero-shot learning \citep{romera2015embarrassingly}. A key cornerstone technology wherein is domain adaptation. This section describes this paradigm.

\textbf{Large-scale pre-trained models.} Recent advances are seen mainly in natural language processing (NLP), particularly after the appearance of transformer \citep{vaswani2017attention}. ELMo \citep{peters2018deep} finds the word embedding in NLP is not invariant in different application domains, but considerably changes with context. Based on this observation, ELMo  pre-trains a large-scale bidirectional LSTM on a large text corpus to generate word vectors by fine-tuning. BERT \citep{devlin2018bert} employs the transformer encoder for detecting bidirectional information in the context. Meanwhile, \citet{liu2018generating} employs the transformer decoder to word embedding with fine-tuning, in order to realize wider attention. GPT \citep{radford2018improving} also employs the transformer decoder but is fine-tuned on each specific task for better performance. Extended from GPT, GPT-2 \citep{radford2019language} and GPT-3 \citep{brown2020language} construct huge models in order to realize zero-shot learning. The comparison between these ``super models'' is presented in the following table.

\begin{table}[ht]
    \centering
    \caption{List of Super Models (SMs)}
    \begin{tabular}{cccc}
    \toprule
    SM & Architecture & Params \\
    \midrule
    ELMo & BiLSTM & - \\
    BERT & Transformer Encoder  & 110M $\sim$ 340M \\
    GPT & Transformer Deconder &  117M \\
    GPT-2 & Transformer Deconder &  117M $\sim$ 1,542M \\
    GPT-3 & Transformer Deconder &  175M \\
    \bottomrule
    \end{tabular}
    \label{tab:plms}
\end{table}
\textbf{Pre-training stage.} The first step of the super-model paradigm is pre-training a super model on large-scale data, sometimes of multi-modality. 
The learned model in this stage is of high quality that the approximation and generalization of the output hypothesis are usually excellent, which suggests that the learned model has stored rich general knowledge in the learned model. This makes it possible to apply the learned model for the smaller specific application domains. 

\textbf{Fine-tuning stage.} The learned model is then fine-tuned on the target domain, usually a smaller specific domain. The stored general knowledge is thereby transferred to the target domain. In this way, super-model paradigm reduces considerable sources of knowledge discovery in the target domain. 

\textbf{Theoretical advantages.} According to the PAC-Bayesian theory, the generalizability of the learned model is determined by the distance between the posterior and the prior. As we will show in the next two sections, the very large-scale training data in the pre-training stage secures learning a high-quality model with no-knowledge prior. The learned knowledge is of high value but consumed enormous resources which is not accessible for many potential machine learning users, particularly small and medium-sized enterprises. In the super-model paradigm, the high-quality model learned in the pre-training stage is employed as the initialization in the fine-tuning stage. In this way, we significantly reduce the needed sample complexity in the fine-tuning stage.

\textbf{Industrial values.} Machine learning has been thriving in a wide range of areas. However, the industrial applications are still limited. This is partially caused by the high cost of computing facilities and data annotations. The paradigm based on super models significantly reduce the cost of machine learning applications. This is particularly important for small and medium-size enterprises. 

\textbf{Climate value.} Super-model paradigm enables recycling discovered general knowledge in enormous application domains. This would also help significantly reduce the carbon emission. Meanwhile, the super-model paradigm centralizes the modeling training process which can help manage the geographic location and the datacenter infrastructure in order to reduce the carbon emission, as a recurrent work suggested \citep{patterson2021carbon}:
\begin{itemize}
\item
Geographic location of machine learning workload scheduling can result carbon emission vary around five times to ten times, even when the country and the organization remain invariant.
\item
Cloud data centers can be around 1.4-2X more energy-efficient. Meanwhile, machine learning-oriented accelerators can be ~2-5X more effective.
\end{itemize}
Super-model paradigm can thus reduce the carbon print of machine learning application and further contribute in slowing down the climate crisis.

\section{Diffusion processes in super-model paradigm}
\label{sec:diffusion}

We consider a diffusion process-based model that serves an envelope for domain adaptation methods. Two diffusion processes are designed for modeling the pre-training and fine-tuning stages, respectively. The knowledge transition can then be modeled via the transition of diffusion processes.

\subsection{Diffusion process in pre-training}

In pre-training, SGD explores on the loss surface for a decent local minimum. Compared with gradient descent, SGD introduces gradient noise into the gradient and then the weight. The noise plays as an implicit regularizer that controls the hypothesis complexity of the learned model. In this section, we employ a stochastic differential equation to characterize the trajectory of SGD.

We assume that the loss function in the local region around the minimum is convex and second-order differentiable, as shown in the following assumption.

\begin{assumption}
Suppose that the empirical risk $\mathcal R(\theta)$ around the optimum as the following equation,
\begin{equation}
\label{eq:local_property}
	\mathcal R(\theta) = \frac{1}{2} \theta^{\top} A_{PT} \theta,
\end{equation}
where $A_{PT}$ is the Hessian matrix around the minimum and is a (semi) positive-definite matrix.
\end{assumption}


\begin{remark}
This assumption implicitly assumes that the converged local minimum is at the zero point. This would not influence the generality under translational motion. Specifically, suppose the converged local minimum is at $\theta'$. We may perform a translational motion to the neural network to move the converged local minimum to zero.
\end{remark}

\begin{remark}
The Hessian matrix $A_{PT}$ of the loss surface characterizes the local geometry around the converged local minimum. Its determinant characterizes the flatness/sharpness of the loss function around the local minimum \citep{keskar2016large, goyal2017accurate}.
\end{remark}

\begin{remark}
The covariance matrix $C_{PT}$ characterizes the fluctuation introduced by the mini bathes into the gradient estimation. A recent intuition for the advantage of SGD is that it introduces noise into the gradient, so that it can jump out of bad local minima. 
\end{remark}

The loss $l_{n}(\theta)$ and gradient $\hat{\mathcal R}(\theta)$ calculated on a mini-batch are un-biased estimators of the empirical risk $\hat{\mathcal R}$ and the full gradient $\nabla_{\theta} \mathcal R(\theta)$, as follows,
\begin{gather}
	\mathbb E \left[ l_{n}(\theta) \right] = \mathbb E \left[ \hat{\mathcal R}(\theta) \right] = \mathcal R(\theta),\\
	 \mathbb E \left[ \nabla_{\theta} l_{n}(\theta) \right] = \mathbb E \left[ \hat{g}_{S}(\theta) \right] = g(\theta) = \nabla_{\theta} \mathcal R(\theta),
\end{gather}
where the expectations are in terms of the corresponding examples $(X, Y)$.

The fluctuations introduced by the mini batches are modeled by Gauss distributions centered at $g(\theta) = \nabla_{\theta} R(\theta)$. Specifically, we assume that
\begin{equation}
\label{eq:single_gradient_noise}
	\nabla_{\theta} l_{n}(\theta) 
	\sim \mathcal N(g(\theta), C),
\end{equation}
where $C$ is the covariance matrix and is a constant matrix for all $\theta$.
This Gaussian assumption is also employed in by \citet{e2017proposal} and \citet{mandt2017stochastic}.
Therefore, we further have the following estimation,
	\begin{align}
	\label{eq:stochastic_gradient_noise}
		\hat g_{S}(\theta) = \frac{1}{|S|} \sum_{n \in S} \nabla_{\theta} l_{n}(\theta) 
		\sim \mathcal N\left(g(\theta), \frac{1}{|S|}C\right).
	\end{align}
	
SGD uses the stochastic gradient $\hat{g}_{S} (\theta)$ to iteratively update the parameter $\theta$ in order to minimize the function $\mathcal R(\theta)$:
\begin{align}
\label{eq:OU_process}
	\Delta \theta(t) = & \theta(t + 1) - \theta(t) 
	= - \eta \hat{g}_{S}(\theta(t)) 
	= - \eta g(\theta) + \frac{\eta}{\sqrt{|S|}}
	B \Delta W,
	\end{align}
	and
	\begin{equation*}
		\Delta W \sim \mathcal N(0, I),
	\end{equation*}
	where $B$ is positive definite matrix which characterizes the covariance of the gradient noise. We define that 
	\begin{equation*}
	C = B^\top B.
	\end{equation*}
In this paper, we consider the case that the batch size $|S|$ and learning rate $\eta$ are constant.

Combining eqs. (\ref{eq:local_property}) and (\ref{eq:OU_process}), we have the following analytic form of the stationary distribution \citep{gardiner1985handbook}:
\begin{equation}
\label{eq:stationary_distribution}
	q_{PT}(\theta) = M_{PT} \exp \left\{ -\frac{1}{2}\theta^{\top} \Sigma_
	{PT}^{-1} \theta \right\},
\end{equation}
where $M_{PT}$ is the normalizer and 
	\begin{gather*}
	\Sigma_{PT} A_{PT} + A_{PT} \Sigma_{PT} = \frac{\eta_{PT}}{|S_{PT}|} C_{PT},
	\end{gather*}
	$\eta_{PT}$, $|S_{PT}|$, and $C_{PT}$ are the learning rate, batch size, and the covariance matrix in the pre-training stage, respectively.


\begin{remark}
In this section, we show that the learned hypothesis is drawn from the steady distribution of a Fokker-Plank equation, which is a Gibs-Boltzmann distribution centered around the zero point.
\end{remark}

\subsection{Knowledge transition in fine-tuning}

SGD in the fine-tuning stage can also be characterized by the Uhlenbeck-Ornstein equation (eq. \ref{eq:OU_process}), while the initial condition is different. The fine-tuning stage is initialized by the steady distribution of the pre-training stage $\Sigma_{PT}$. Similarly, the SGD converges to another steady distribution $\Sigma_{FT}$.
In this way, we model the domain adaptation as a two-stage diffusion process. The second-stage diffusion process characterizes the knowledge transition between the two domains.

We assume that the loss function in the local region around the minimum is convex and $2$-order differentiable, as shown in the following assumption.
\begin{assumption}
Suppose that the empirical risk $\mathcal R(\theta)$ around the optimum as the following equation,
\begin{equation}
\label{eq:local_property_FT}
	\mathcal R(\theta) = \frac{1}{2} (\theta - \theta_{FT})^{\top} A_{FT} (\theta - \theta_{FT}),
\end{equation}
where $A$ is the Hessian matrix around the minimum and is a (semi) positive-definite matrix.
\end{assumption}

Recall that we assumed that the converged local minimum in the pre-training stage is at the zero point. In the fine-tuning stage, the converged local minimum cannot be assumed at the same point in general. Thus, a shift term $\theta_{FT}$ is introduced to characterize the the shift of the converged local minimum.

Similarly, combining eqs. (\ref{eq:local_property_FT}) and (\ref{eq:OU_process}), we have the following analytic form of the stationary distribution:
\begin{equation}
\label{eq:stationary_distribution_FT}
	q_{FT}(\theta) = M_{PT} \exp \left\{ -\frac{1}{2}(\theta - \theta_{FT})^{\top} \Sigma_
	{FT}^{-1} (\theta - \theta_{FT}) \right\},
\end{equation}
where $M_{FT}$ is the normalizer and 
	\begin{gather*}
	\Sigma_{FT} A_{FT} + A_{FT} \Sigma_{FT} = \frac{\eta_{PT}}{|S_{PT}|} C_{PT},
	\end{gather*}
	$\eta_{FT}$, $|S_{FT}|$, and $C_{FT}$ are the learning rate, batch size, and the covariance matrix in the fine-tuning stage.


Recall that the converged local minimizer in the pre-training stage is drawn from a Gibs-Boltzmann distribution centered at the zero point. This is inherited from the assumption that the local minimum is around the zero point. However, in the fine-tuning stage, the converged local minimum has a shift $\theta_{FT}$ from the zero point. This leads to a shift $\theta_{FT}$ of the distribution of the learned hypothesis.

\section{Generalization analysis of super-model paradigm}
\label{sec:generalization}

The knowledge transition characterized by the diffusion process in the fine-tuning. In this paper, we employ the PAC-Bayesian theory to analyze the generalizability of domain adaptation.

\subsection{PAC-Bayesian framework}

PAC-Bayesian theory corporates the PAC theory and Bayesian statistics \citep{mcallester1999pac, mcallester1999some}. It presents a generalization bound for a stochastic algorithm based on the distance between the learned hypothesis and the prior measured by the KL divergence. The PAC-Bayesian bound characterizes the trade-off between minimising the empirical risk and exploring further areas of the hypothesis space from the initial.



\begin{lemma}[see \citet{mcallester1999pac}, Theorem 1]
\label{lemma:PAC-Bayesian}
	For any positive real $\delta \in (0, 1)$, with probability at least $1 - \delta$ over a sample of size $N$, we have the following inequality for all distributions $Q$:
	\begin{align}
	\label{eq:PAC-Bayesian}
		\mathcal R(Q) \le & \hat{\mathcal R}(Q) + \sqrt{\frac{\mathcal D(Q || P) + \log \frac{1}{\delta} + \log N + 2}{2N - 1}},
	\end{align}
	where $\mathcal D(Q || P)$ is the KL divergence between the distributions $Q$ and $P$ and is defined as,
\begin{equation}
\label{eq:def_KL_divergence}
	\mathcal D(Q || P) = \mathbb E_{\theta \sim Q} \left( \log \frac{Q(\theta)}{P(\theta)} \right).
\end{equation}
\end{lemma}

This lemma characterizes the influence on the generalization via the distance between the distribution $Q$ of the learned hypothesis and the prior $P$ measured by the KL divergence $\mathcal D(Q || P)$. The KL divergence serves as a hypothesis complexity measure. In specific, a larger KL divergence corresponds to a larger hypothesis complexity and further a worse generalizability.

\subsection{Generalization bound}

We then obtain a generalization bound for the pre-training stage as follows.

\begin{theorem}
\label{thm:PAC-Bayesian_SGD-PT}
	For any positive real $\delta \in (0, 1)$, with probability at least $1 - \delta$ over a training sample set of size $N_{PT}$, we have the following inequality for the distribution $Q$ of the output hypothesis function of SGD:
	\begin{align}
	\label{eq:thm:PAC-Bayesian_SGD-PT}
		& R(Q_{PT}) \le \hat R(Q_{PT}) 
		+ \sqrt{\frac{D(Q_{PT}, P) + 2 \log \left( \frac{1}{\delta} \right) + 2 \log N_{PT} + 4}{4N_{PT} - 2}},
	\end{align}
	where
\begin{align*}
D(Q_{PT}, P) 
= \log \left({\det(\Sigma_{PT})} \right) + \text{tr}(\Sigma_{PT} - I).
\end{align*}
\end{theorem}

The proof for this generalization bound has two parts: (1) utilize results from stochastic differential equation (SDE) to find the stationary solution of the latent Ornstein-Uhlenbeck process (eq. \ref{eq:OU_process}) which expresses the iterative update of SGD; and (2) adapt the PAC-Bayes framework to obtain the generalization bound based on the stationary distribution. A detailed proof is omitted here and is given in Appendix \ref{sec:proof_PAC-Bayesian_SGD}.

Similarly, we can obtain a generalization bound for the fine-tuning stage as follows.

\begin{theorem}
\label{thm:PAC-Bayesian_SGD-FT}
	For any positive real $\delta \in (0, 1)$, with probability at least $1 - \delta$ over a training sample set of size $N_{FT}$, we have the following inequality for the distribution $Q_{FT}$ of the output hypothesis function of SGD:
	\begin{align*}
	\label{eq:thm:PAC-Bayesian_SGD-FT}
		& R(Q_{FT}) \le \hat R(Q_{FT}) 
		+ \sqrt{\frac{D(Q_{FT}, Q_{PT}) + 2 \log \left( \frac{1}{\delta} \right) + 2 \log N_{FT} + 4}{4N_{FT} - 2}},
	\end{align*}
	where
\begin{align*}
& D(Q_{FT}, Q_{PT}) 
= \log \left({\det(\Sigma_{PT}^{-1}\Sigma_{FT})} \right) + \text{tr}(\Sigma_{PT}^{-1}\Sigma_{FT} - I) + \theta_{FT}^\top \Sigma_{PT}^{-1}\theta_{FT}.
\end{align*}
\end{theorem}

\begin{remark}
The generalization bounds in both pre-training and fine-tuning stages are in order of $\mathcal O(1/\sqrt{N})$, which suggests that the generalization error converges to zero when the training sample size goes to infinity.
\end{remark}


\subsection{Dominance of fine-tuning in generalization of domain adaptation}

In the super-model paradigm, the model is usually pre-trained on large amounts of data in a wide source domain and then fine-tuned on specific domains with relatively smaller training data. The training sample size $N_{PT}$ in the source domain is significantly larger than the size $N_{FT}$ of the data in the target domain. For example, the GPT-3 is trained on 45TB data. Meanwhile, the training sample size in the target domain is relatively smaller. 

\begin{remark}
Combining Theorems \ref{thm:PAC-Bayesian_SGD-PT} and \ref{thm:PAC-Bayesian_SGD-FT}, the comparison between the training sample sizes on the source domain and the target domain suggests that the generalization error of the fine-tuning stage is dominant in the super-model paradigm.
\end{remark}

\subsection{Impact of the domain shifts}

Theorem \ref{thm:PAC-Bayesian_SGD-FT} helps characterize how the domain shifts between the source domain and the target domain influences the generalization on the target domain. The domain shifts are measured by the following discrepancy.

\begin{definition}[Domain discrepancy]
\label{definition:domain_discrepancy}
Suppose the distributions of the learned models in the pre-training and fine-tuning are $Q_{FT}$ and $Q_{PT}$ as follows,
\begin{align}
	q_{PT}(\theta) = & M_{PT} \exp \left\{ -\frac{1}{2}\theta^{\top} \Sigma_	{PT}^{-1} \theta \right\},\nonumber\\
	q_{FT}(\theta) = & M_{FT} \exp \left\{ -\frac{1}{2}(\theta - \theta_{FT})^{\top} \Sigma_
	{FT}^{-1} (\theta - \theta_{FT}) \right\},
\end{align}
where $M_{PT}$ and $M_{FT}$ are two normalizers, $\Sigma_{PT}$ and $\Sigma_{FT}$ are two covariance matrices, and $\theta_{FT}$ is the center shift between the two learned hypotheses.

Then, the domain discrepancy between the two domains are defined as below,
\begin{align*}
& D(Q_{FT}, Q_{PT}) 
= \log \left({\det(\Sigma_{PT}^{-1}\Sigma_{FT})} \right) + \text{tr}(\Sigma_{PT}^{-1}\Sigma_{FT} - I) + \theta_{FT}^\top \Sigma_{PT}^{-1}\theta_{FT}.
\end{align*}
\end{definition}

\begin{remark}
In Definition \ref{definition:domain_discrepancy}, we assume the distribution of the pre-trained model is centered at the zero point.
This assumption would not hurt the generality. Suppose the distribution center is not at the zero point. One may move it to the zero point via reparamterization.
\end{remark}

\begin{remark}
The domain discrepancy $D(Q_{FT}, Q_{PT})$ is constituted by two parts: (1) $\Sigma_{PT}^{-1}\Sigma_{FT}$ characterizes the matchness in the aspect of covariance; and (2) $\theta_{FT}^\top \Sigma_{PT}^{-1}\theta_{FT}$ characterizes the matchness of the center shift in the lens of the the covariance in the source domain.
\end{remark}

\begin{remark}
Our generalization bound suggests that the generalization on the target domain is determined by the magnitude of the domain shifts. Generally, larger domain shifts lead to worse generalization on the target domain. This supports the heuristic in practice that the performance on the target domain is limited by the domain shifts.
\end{remark}

Based on Definition \ref{definition:domain_discrepancy}, one can get the following lemma.

\begin{lemma}
\label{lemma:domain_discrepancy}
The domain discrepancy $D(Q_{FT}, Q_{PT})$ can be rearranged as follows,
\begin{align*}
& D(Q_{FT}, Q_{PT}) \nonumber\\
\le & \log(\text{tr}(\Sigma_{PT}^{-1}\Sigma_{FT})) + \text{tr}(\Sigma_{PT}^{-1}\Sigma_{FT}) + \theta_{FT}^\top \Sigma_{PT}^{-1}\theta_{FT} + d \log d - d.
\end{align*}
\end{lemma}

\begin{proof}[Proof of Lemma \ref{lemma:domain_discrepancy}]
We have that
\begin{align*}
& D(Q_{FT}, Q_{PT}) \nonumber\\
= & \log \left({\det(\Sigma_{PT}^{-1}\Sigma_{FT})} \right) + \text{tr}(\Sigma_{PT}^{-1}\Sigma_{FT} - I) + \theta_{FT}^\top \Sigma_{PT}^{-1}\theta_{FT} \nonumber\\
\le & \log (d^d \text{tr}(\Sigma_{PT}^{-1}\Sigma_{FT})) - d + \text{tr}(\Sigma_{PT}^{-1}\Sigma_{FT}).
\end{align*}
\end{proof}

Based on Lemma \ref{lemma:domain_discrepancy}, we define a new notion for measuring the domain shifts as follows.

\begin{definition}[Dimension-dependent domain discrepancy]
\label{definition:domain_discrepancy_II}
Suppose the distributions of the learned models in the pre-training and fine-tuning are $Q_{FT}$ and $Q_{PT}$ as follows,
\begin{align}
	q_{PT}(\theta) = & M_{PT} \exp \left\{ -\frac{1}{2}\theta^{\top} \Sigma_	{PT}^{-1} \theta \right\},\nonumber\\
	q_{FT}(\theta) = & M_{FT} \exp \left\{ -\frac{1}{2}(\theta - \theta_{FT})^{\top} \Sigma_
	{FT}^{-1} (\theta - \theta_{FT}) \right\},
\end{align}
where $M_{PT}$ and $M_{FT}$ are two normalizers, $\Sigma_{PT}$ and $\Sigma_{FT}$ are two covariance matrices, and $\theta_{FT}$ is the center shift between the two learned hypotheses.

Then, the domain discrepancy between the two domains are defined as below,
\begin{align*}
& \tilde D(Q_{FT}, Q_{PT}) \nonumber\\
= & \log(\text{tr}(\Sigma_{PT}^{-1}\Sigma_{FT})) + \text{tr}(\Sigma_{PT}^{-1}\Sigma_{FT}) + \theta_{FT}^\top \Sigma_{PT}^{-1}\theta_{FT} + d \log d - d.
\end{align*}
\end{definition}

From Theorem \ref{thm:PAC-Bayesian_SGD-FT}, one may obtain the following corollary.

\begin{corollary}
\label{cor:PAC-Bayesian_SGD-FT}
	For any positive real $\delta \in (0, 1)$, with probability at least $1 - \delta$ over a training sample set of size $N_{FT}$, we have the following inequality for the distribution $Q_{FT}$ of the output hypothesis function of SGD:
	\begin{align}
	\label{eq:thm:PAC-Bayesian_SGD-FT}
		& R(Q_{FT}) \le \hat R(Q_{FT}) 
		+  \sqrt{\frac{\tilde D(Q_{FT}, Q_{PT}) + 2 \log \left( \frac{1}{\delta} \right) + 2 \log N_{FT} + 4}{4N_{FT} - 2}},
	\end{align}
	where
\begin{align*}
& \tilde D(Q_{FT}, Q_{PT}) \nonumber\\
= & \log(\text{tr}(\Sigma_{PT}^{-1}\Sigma_{FT})) + \text{tr}(\Sigma_{PT}^{-1}\Sigma_{FT}) + \theta_{FT}^\top \Sigma_{PT}^{-1}\theta_{FT} + d \log d - d.
\end{align*}
and $A$ is the Hessian matrix of the loss function around the local minimum.
\end{corollary}

\section{Discussion and future work}
\label{sec:discussion}

Large-scale pre-trained models, such as GPT-3 and Bert, enables a new industrial paradigm: pre-training a super model on large amounts of multi-modality data (sometimes of low-quality) and then fine-tuning the learned model to smaller specific application domains. This paradigm may start a {\it super-model paradigm} that would significantly reduce the application cost of machine learning, which is critical for enormous small and medium-sized enterprises.

A major technique in this paradigm is domain adaptation which enables the knowledge transfer between the two domains. We model a super-model paradigm as a two-stage diffusion process: (1) in the pre-training stage, the trajectory of the stochastic gradient descent (SGD) or its variants searches on the loss surface driven by Uhlenbeck-Ornstein process discretely or smoothly by the Fokker-Plank equation. The model weight starts from a no-knowledge prior and converges to a Maxwell-Boltzmann distribution; and (2) in the fine-tuning stage, the trajectory of SGD is driven by a similar SDE, which starts from the learned model distribution in the pre-training stage and converges to another Maxwell-Boltzmann distribution. Based on the diffusion processes, an $\mathcal O(1/\sqrt{N})$ generalization bound is obtained via the PAC-Bayesian framework. 

The generalization bounds suggest that the fine-tuning stage dominates the generalization of the whole paradigm. The generalization is determined by the domain discrepancy between the pre-training and fine-tuning domains, which is characterized by a new measure based on the covariance and domain shifts.

In this work, we make several assumptions and abstractions. This section discusses the limitation introduced by them and give several potential extensions.

\begin{itemize}
\item
\textbf{Model compression in the fine-tuning stage.} In this paper, we ignore the model compression approaches in the fine-tuning stage, which are sometimes employed in practice. Popular model compression methods include model distillation, pruning, and quantization. The effects of model compression can be seen as operators on the loss surface and the learned model. A future direction is to mathematically characterize the influence of model compression. The results may be plug-and-play components to the presented theory in this paper.

\item
\textbf{Gradient noise in SGD.} In this paper, we assume that the gradient noise is drawn from a Gaussian distribution. Recent works also made assumptions that the gradient noise as a Levy process, Laplacian noise, etc. The exact distribution of the gradient noise is still an open problem. In addition, the gradient noise is assumed state-independent, which can be easily extended to state-dependent. A future direction is to study the distribution of the gradient noise in SGD. It is worth noting that relatively little efforts are needed to change the gradient noise distribution assumptions in this paper. 

\item
\textbf{Advanced techniques in modeling SGD.} In this paper, we model the trajectory of SGD via Fokker-Planck equation and Uhlenbeck-Ornstein process. This modeling ignores the influence of several techniques, such as momentum and adaptive learning rate. Recent works discover that these techniques may have implicit regularization on the learned model while would not have determinant impact. A future direction is modeling the SGD as a more sophisticated stochastic differential equation.

\item
\textbf{Distribution/data-dependent priors.} Some works design priors relying on the data generation distribution but still not directly relying on the training data. This would be reasonable since we can assume the data distribution has been fixed before the data was collected \citep{lever2013tighter}. Such {\it distribution-dependent} priors have shown to be able to considerably tighten the generalization bounds. \citet{negrea2019information} further push the frontier that constructs priors not independent with data. Suppose $S_j$ is a subset of $\subset S$ with size of $n < m$. One may design a prior exploiting $S_J$ to deliver a data-dependent forecast of the posterior $Q$. A future direction is modeling the SGD via distribution/data-dependent priors.
\end{itemize}

\section{Proofs}
\label{sec:proof}

This section presents the proofs for the given theory.

We model the the iterative updates in SGD employing a stochastic differential equation. This approach is also seen in the literature; see, e.g., \citet{e2017proposal, mandt2017stochastic, mou2017generalization, he2019control, meng2020dynamic, cheng2020stochastic, xie2020diffusion, wang2021learning}.

We first translate the updates in SGD as Ornstein-Uhlenbeck process \citep{uhlenbeck1930theory} under some mild assumptions. The Ornstein-Uhlenbeck process has a steady distribution which is then employed to characterize the distribution of the learned hypothesis. We further obtain a generalization bound via PAC-Bayesian framework by exploiting the stationary distribution, which characterizes the influence on the generalization via the distance between the output hypothesis distribution and its prior \citep{mcallester1999pac, mcallester1999some}.

\subsection{Proof of Theorem \ref{thm:PAC-Bayesian_SGD-PT}}
\label{sec:proof_PAC-Bayesian_SGD}

The proof for Theorem \ref{thm:PAC-Bayesian_SGD-PT} replies on the following lemma.

\begin{lemma}[cf. \citet{mandt2017stochastic}, pp. 27-18, Appendix B]
	Under the second-order differentiable assumption (eq. \ref{eq:local_property}), the Ornstein-Uhlenbeck process (eq. \ref{eq:OU_process})'s stationary distribution,
\begin{equation}
\label{eq:stationary_distribution}
	q(\theta) = M \exp \left\{ -\frac{1}{2}\theta^{\top} \Sigma_{PT}^{-1} \theta \right\},
\end{equation}
	has the following property,
	\begin{equation}
	\label{eq:covariance}
		A \Sigma_{PT} + \Sigma_{PT} A = \frac{\eta}{|S|} C.
	\end{equation}
\end{lemma}

This lemma gives the analytic form of the steady distribution of the Ornstein-Uhlenbeck process. This lemma is from \citet{mandt2017stochastic}. Here, we recall the proof to make this paper complete.

\begin{proof}
\label{}
	Form a result in Ornstein-Uhlenbeck process \citep{gardiner1985handbook}, we know that the parameter $\theta$ has the following analytic solution,
	\begin{equation}
		\theta(t) = \theta(0)e^{-At} + \sqrt{\frac{\eta}{|S|}} \int_{0}^{t} e^{-A(t - t')} B \text{d} W(t'),
	\end{equation}
	where $W(t')$ is a white noise and follows $\mathcal N(0, I)$.
	From eq. (\ref{eq:stationary_distribution}), we know that
	\begin{equation}
		\Sigma_{PT} = \mathbb E_{\theta \sim Q} \left[\theta \theta^{\top}\right].
	\end{equation}
	Therefore, we have the following equation,
	\begin{align}
		A \Sigma_{PT} + \Sigma_{PT} A = & \frac{\eta}{|S|} \int_{-\infty}^{t} A e^{-A(t - t_{0})} C e^{-A(t - t_{0})} \text{d} t' \nonumber\\
		& + \frac{\eta}{|S|} \int_{-\infty}^{t} e^{-A(t - t_{0})} C e^{-A(t - t_{0})} \text{d} t' A \nonumber\\
		= & \frac{\eta}{|S|} \int_{-\infty}^{t} \frac{\text{d}}{\text{d}t'} A e^{-A(t - t_{0})} C e^{-A(t - t_{0})} \nonumber\\
		= & \frac{\eta}{|S|} C.
	\end{align}
	The proof is completed.
\end{proof}
	
Then, we can prove Theorem \ref{thm:PAC-Bayesian_SGD-PT}. This proof is inspired by \citet{he2019control}. Here, we recall the proof to make this paper complete.

\begin{proof}[Proof of Theorem \ref{thm:PAC-Bayesian_SGD-PT}]
\label{proof:PAC-Bayesian_SGD}
	In PAC-Bayesian framework (Lemma \ref{lemma:PAC-Bayesian}), an essential part is the KL divergence between the distribution of the learned hypothesis and the priori on the hypothesis space. The prior distribution can be interpreted as the distribution of the initial parameters, which are usually settled according to Gaussian distributions or uniform distributions.\footnote{Usually, when there is no confident prior knowledge of the latent model parameters, the priori should be set as distributions with no information, such as Gaussian distributions or uniform distributions. This setting comes from two considerations: (1) Once the algorithms based on the Bayesian statistics can converge, after long enough time and with big enough data, the algorithms can always converge to the stationary distributions. This is guaranteed by the assumption that the stationary solution of the latent stochastic differential equation exists and is unique; (2) Setting priori should be very careful, as we can not assume we have any knowledge of the target hypothesis function before we have started training the model.} 
Here, we use a standard Gaussian distribution $\mathcal N(0, I)$ as the priori.
Suppose the densities of the stationary distribution $Q_{PT}$ and the prior distribution $P$ are respectively $p(\theta)$ and $q_{PT}(\theta)$ in terms of the parameter $\theta$ as the following equations,
	\begin{gather}
	\label{eq:density_prior}
		p(\theta) = \frac{1}{\sqrt{2\pi\det(I)}} \exp \left\{ -\frac{1}{2} \theta^{\top} I \theta \right\},\\
	\label{eq:density_stationary}
		q_{PT}(\theta) = \frac{1}{\sqrt{2\pi\det(\Sigma_{PT})}} \exp \left\{ -\frac{1}{2} \theta^{\top} \Sigma_{PT}^{-1} \theta \right\},
	\end{gather}
where ep. (\ref{eq:density_stationary}) comes from eq. (\ref{eq:stationary_distribution}) by calculating the normalizer $M$.

Therefore,
	\begin{align}
	\label{eq:prob_ratio}
		& \log \left( \frac{q_{PT}(\theta)}{p(\theta)} \right) \nonumber\\
		= & \log \left( \frac{\sqrt{2\pi\det(I)}}{\sqrt{2\pi\det(\Sigma_{PT})}} \exp \left\{ \frac{1}{2} \theta^{\top} I \theta -\frac{1}{2} \theta^{\top} \Sigma_{PT}^{-1} \theta \right\} \right) \nonumber\\
		= & \frac{1}{2} \log \left( \frac{1}{\det(\Sigma_{PT})} \right) + \frac{1}{2} \left( \theta^{\top} I \theta - \theta^{\top} \Sigma_{PT}^{-1} \theta \right).
	\end{align}
	
Applying eq. (\ref{eq:prob_ratio}) 
to eq. (\ref{eq:def_KL_divergence}), we can calculate the KL divergence between the distributions $Q_{PT}$ and $P$ (we assume $\Theta = \mathbb R^{d}$):
	\begin{align}
	\label{eq:proof:KL}
		& \mathcal D(Q_{PT} || P) \nonumber\\
		= & \mathbb E_{\theta \sim Q_{PT}} \left( \log \frac{Q_{PT}(\theta)}{P(\theta)} \right) \nonumber\\
		= & \int_{\theta \in \Theta} \log \left( \frac{q_{PT}(\theta)}{p(\theta)} \right) q_{PT}(\theta) \text{d} \theta \nonumber\\
		= & \int_{\theta \in \Theta} \left[ \frac{1}{2} \log \left( \frac{1}{\det(\Sigma_{PT})} \right) + \frac{1}{2} \left( \theta^{\top} I \theta - \theta^{\top} \Sigma_{PT}^{-1} \theta \right) \right] q(\theta) \text{d} \theta \nonumber\\
		= & \frac{1}{2} \log \left( \frac{1}{\det(\Sigma_{PT})} \right) + \frac{1}{2} \int_{\theta \in \Theta} \theta^{\top} I \theta p(\theta) \text{d} \theta - \frac{1}{2} \int_{\mathbb R^{|S|}} \theta^{\top} \Sigma_{PT}^{-1} \theta q(\theta) \text{d} \theta \nonumber\\
		= & \frac{1}{2} \log \left( \frac{1}{\det(\Sigma_{PT})} \right) + \frac{1}{2} \mathbb E_{\theta \sim \mathcal N(0, \Sigma_{PT})} \theta^{\top} I \theta - \frac{1}{2} \mathbb E_{\theta \sim \mathcal N(0, \Sigma_{PT})} \theta^{\top} \Sigma_{PT}^{-1} \theta\nonumber\\
		= & \frac{1}{2} \log \left( \frac{1}{\det(\Sigma_{PT})} \right) + \frac{1}{2} \text{tr}(\Sigma_{PT} - I). 
	\end{align}
	
	From eq. (\ref{eq:covariance}), we have that
	\begin{gather}
		A_{PT} \Sigma_{PT} + \Sigma_{PT} A_{PT} = \frac{\eta_{PT}}{|S_{PT}|} C.
	\end{gather}
	Therefore,
	\begin{gather}
		A_{PT} \Sigma_{PT} A_{PT}^{-1} + \Sigma_{PT} = \frac{\eta_{PT}}{|S_{PT}|} C A_{PT}^{-1}.
	\end{gather}
	After calculating the trace of the both sides, we have the following equation,
	\begin{gather}
		\text{tr}\left( A_{PT} \Sigma_{PT} A_{PT}^{-1} + \Sigma_{PT} \right) = \text{tr}\left( \frac{\eta_{PT}}{|S_{PT}|} C A_{PT}^{-1} \right).
	\end{gather}
	The left-hand side (LHS) is as follows,
	\begin{align}
	\text{LHS} = & \text{tr}\left( A_{PT} \Sigma_{PT} A_{PT}^{-1} + \Sigma_{PT} \right) \nonumber\\
	= & \text{tr}\left( A_{PT} \Sigma_{PT} A_{PT}^{-1} \right) + \text{tr}\left( \Sigma_{PT} \right) \nonumber\\
	= & \text{tr}\left( \Sigma_{PT} A_{PT}^{-1} A_{PT} \right) + \text{tr}\left( \Sigma_{PT} \right) \nonumber\\
	= & \text{tr}\left( \Sigma_{PT} \right) + \text{tr}\left( \Sigma_{PT} \right) \nonumber\\
	= & 2 \text{tr}\left( \Sigma_{PT} \right).
	\end{align}
	Therefore,
	\begin{align}
	\label{eq:trace_covariance}
	& \text{tr}\left( \Sigma_{PT} \right) 
	=  \frac{1}{2} \text{tr}\left( \frac{\eta_{PT}}{|S_{PT}|} C A_{PT}^{-1} \right) 
	= \frac{1}{2} \frac{\eta_{PT}}{|S_{PT}|} \text{tr}\left( C A_{PT}^{-1} \right).
	\end{align}

	At the same time, we can easily calculate that
	\begin{equation}
	\label{eq:trace}
		\text{tr}(I) = d,
	\end{equation}
	as $I \in \mathbb R^{d \times d}$, where $d$ is the dimension of the parameter $\theta$.
	
	Insert eqs. (\ref{eq:trace_covariance}) and (\ref{eq:trace}) to eq. (\ref{eq:proof:KL}), we can get the following inequality,
	\begin{align}
	\label{eq:KL_divergence_SGD}
		& \mathcal D(Q_{PT} || P) 
		\le \frac{1}{4} \frac{\eta_{PT}}{|S_{PT}|} tr(C A_{PT}^{-1}) - \frac{1}{2} \log (\det (\Sigma_{PT})) - \frac{1}{2}d. 
	\end{align}
	
	Eq. (\ref{eq:KL_divergence_SGD}) gives an upper bound for the distance (measured by KL divergence) between the stationary distribution of the output weights by SGD and the priori on the hypothesis space. Considering the monotonicity of the generalization bound in terms of the KL divergence, we can further obtain a PAC-Bayesian generalization bound for SGD by inserting the KL divergence bound (eq. \ref{eq:KL_divergence_SGD}) into the PAC-Bayesian framework (eq. (\ref{eq:PAC-Bayesian}) of Lemma \ref{lemma:PAC-Bayesian}).
	
	
	
	
	The proof is completed.
\end{proof}

\subsection{Proof of Theorem \ref{thm:PAC-Bayesian_SGD-FT}}
\label{sec:proof_PAC-Bayesian_SGD}
	
This section proves Theorem \ref{thm:PAC-Bayesian_SGD-FT}. The proof is similar to the previous theorem.

\begin{proof}[Proof of Theorem \ref{thm:PAC-Bayesian_SGD-FT}]
\label{proof:PAC-Bayesian_SGD}

Similarly, the distribution $Q_{FT}$ of the learned hypothesis and the prior distributions $Q_{PT}$ are respectively $q_{PT}(\theta)$ and $p_{FT}(\theta)$ in terms of the parameter $\theta$ as the following equations,
	\begin{gather}
	\label{eq:density_prior_FT}
		q_{PT}(\theta) = \frac{1}{\sqrt{2\pi\det(\Sigma_{PT})}} \exp \left\{ -\frac{1}{2} \theta^{\top} \Sigma_{PT}^{-1} \theta \right\},\\
	\label{eq:density_stationary_FT}
		q_{FT}(\theta) = \frac{1}{\sqrt{2\pi\det(\Sigma_{FT})}} \exp \left\{ -\frac{1}{2} \theta^{\top} \Sigma_{FT}^{-1} \theta \right\},
	\end{gather}
where ep. (\ref{eq:density_stationary_FT}) comes from calculating the normalizer $M$.

Therefore,
	\begin{align}
	\label{eq:prob_ratio}
		& \log \left( \frac{q_{FT}(\theta)}{q_{PT}(\theta)} \right) \nonumber\\
		= & \log \left( \frac{\sqrt{2\pi\det(\Sigma_{PT})}}{\sqrt{2\pi\det(\Sigma_{FT})}} \exp \left\{\frac{1}{2} \theta^{\top} \Sigma_{PT}^{-1} \theta - \frac{1}{2} \theta^{\top} \Sigma_{FT}^{-1} \theta \right\} \right) \nonumber\\
		= & \frac{1}{2} \log \left( \frac{\det(\Sigma_{PT})}{\det(\Sigma_{FT})} \right) + \frac{1}{2} \left(\theta^{\top} \Sigma_{PT}^{-1} \theta - \theta^{\top} \Sigma_{FT}^{-1} \theta \right).
	\end{align}
	
	Then, the KL divergence between the distributions $Q_{FT}$ and $Q_{PT}$ are as follows (we assume $\Theta = \mathbb R^{d}$):
	\begin{align}
	\label{eq:proof:KL}
		& \mathcal D(Q_{FT} || Q_{PT}) \nonumber\\
		= & \mathbb E_{\theta \sim Q_{FT}} \left( \log \frac{Q_{FT}(\theta)}{Q_{PT}(\theta)} \right) \nonumber\\
		= & \int_{\theta \in \Theta} \log \left( \frac{q_{FT}(\theta)}{q_{PT}(\theta)} \right) q_{FT}(\theta) \text{d} \theta \nonumber\\
		= & \int_{\theta \in \Theta} \left[ \frac{1}{2} \log \left( \frac{\det(\Sigma_{PT})}{\det(\Sigma_{FT})} \right) + \frac{1}{2} \left(\theta^{\top} \Sigma_{PT}^{-1} \theta - \theta^{\top} \Sigma_{FT}^{-1} \theta \right) \right] q(\theta) \text{d} \theta \nonumber\\
		= & \frac{1}{2} \log \left( \frac{\det(\Sigma_{PT})}{\det(\Sigma_{FT})} \right) + \frac{1}{2} \int_{\mathbb R^{|S|}} \theta^{\top} \Sigma_{PT}^{-1} \theta q_{FT}(\theta) \text{d} \theta \nonumber\\
		& - \frac{1}{2} \int_{\mathbb R^{|S|}} \theta^{\top} \Sigma_{FT}^{-1} \theta q_{FT}(\theta) \text{d} \theta \nonumber\\
		= & \frac{1}{2} \log \left( \frac{\det(\Sigma_{PT})}{\det(\Sigma_{FT})} \right) + \frac{1}{2} \mathbb E_{\theta \sim \mathcal N(0, \Sigma_{PT})} \theta^{\top} \Sigma_{PT}^{-1} \theta - \frac{1}{2} \mathbb E_{\theta \sim \mathcal N(0, \Sigma_{FT})} \theta^{\top} \Sigma_{FT}^{-1} \theta \nonumber\\
		= &\frac{1}{2} \log \left({\det(\Sigma_{PT}^{-1}\Sigma_{FT})} \right) + \frac{1}{2} \text{tr}(\Sigma_{PT}^{-1}\Sigma_{FT} - I) + \frac{1}{2} \theta_{FT}^\top \Sigma_{PT}^{-1}\theta_{FT}.
	\end{align}

Therefore, we have
	\begin{align}
	\label{eq:thm:PAC-Bayesian_SGD-FT}
		& R(Q_{FT}) \le \hat R(Q_{FT})
		+ \sqrt{\frac{D(Q_{FT}, Q_{PT}) + 2 \log \left( \frac{1}{\delta} \right) + 2 \log N_{FT} + 4}{4N_{FT} - 2}},
	\end{align}
	where
\begin{align*}
& D(Q_{FT}, Q_{PT})
= \log \left({\det(\Sigma_{PT}^{-1}\Sigma_{FT})} \right) + \text{tr}(\Sigma_{PT}^{-1}\Sigma_{FT} - I) + \theta_{FT}^\top \Sigma_{PT}^{-1}\theta_{FT},
\end{align*}
		
	The proof is completed.
\end{proof}


\section*{Acknowledgments}

The authors appreciate Shiye Lei for helpful discussions.

\bibliographystyle{plainnat}
\bibliography{bib}
\end{document}